\documentclass{article}
\PassOptionsToPackage{sort,numbers}{natbib}

\usepackage[preprint]{neurips_2026}

\usepackage[utf8]{inputenc} 
\usepackage[T1]{fontenc}    
\usepackage{hyperref}       
\usepackage{url}            
\usepackage{booktabs}       
\usepackage{amsfonts}       
\usepackage{nicefrac}       
\usepackage{microtype}      
\usepackage{xcolor}         
\usepackage{amsmath}
\usepackage{amsthm}
\usepackage{cleveref}
\usepackage{graphicx} 
\usepackage{enumitem}
\usepackage{tcolorbox}

\newtheorem{definition}{Definition}

\newtheorem{theorem}{Theorem}
\newtheorem{corollary}{Corollary}

\DeclareMathOperator*{\argmax}{arg\,max}

\title{How Data Shapes RoPE Frequency Usage: From Positional Scale Matching to Length Generalization}

\author{
Xinyi Wu$^{\ast1}$\qquad 
Siyuan Liu$^{\ast\dagger2}$\qquad 
Ali Jadbabaie$^{1}$\\
$^1$MIT IDSS \qquad $^2$IIIS, Tsinghua University\\
\texttt{\{xinyiwu,jadbabai\}@mit.edu}
\quad \texttt{liusiyua23@mails.tsinghua.edu.cn}}

\begin{document}

\maketitle

\begingroup
\renewcommand{\thefootnote}{\fnsymbol{footnote}}
\footnotetext[1]{Equal contribution. \quad $^\dagger$ Work done during visit at MIT.}
\endgroup

\vspace{-2ex}
\begin{abstract}
Rotary Position Embeddings (RoPE) provide transformers with a fixed grid of positional frequencies, yet trained models use these frequencies highly non-uniformly. We study what determines this frequency usage and propose a data-centered explanation: RoPE frequencies are selected to match the relative-distance structure of the training data. Viewing each frequency as a positional lens, we formalize a field-resolution tradeoff and show that, for a data-induced dependency profile of width \(W\), the optimal frequency scales as \(1/W\). This frequency-matching principle explains controlled observations on synthetic and text-based data, and suggests that the mid-low frequency bands observed in language models arise from the multi-scale dependency structure of natural language. We further connect frequency selection to position-interpolation-based length generalization: scaling frequencies down expands the effective field while reducing resolution. This helps when longer-context dependencies are approximate dilations of those seen during training, but can fail when relevant dependencies do not scale with context length. Empirically, we show that natural language exhibits approximate self-similarity across positional scales, explaining why test-time frequency scaling can support long-context generalization. Overall, our results identify a data-driven mechanism behind emergent RoPE frequency usage and show that long-context generalization depends on two forms of scale matching: between learned frequencies and training-time dependencies, and between frequency scaling and how those dependencies extend to longer contexts. 
\end{abstract}


\vspace{-3ex}
\section{Introduction}
\vspace{-1ex}

A longstanding view in machine learning is that generalization depends on inductive biases, or preferences that prioritize certain solutions over others when data is finite~\cite{Mitchell1980NeedFB,Battaglia2018RelationalIB}. One of the central questions in modern machine learning is how such biases arise in models designed to be broadly flexible~\cite{Tancik2020FourierFL,Wu2025OnTE}. Attention is a striking example~\cite{Vaswani2017AttentionIA,Bahdanau2014NeuralMT,Kim2017StructuredAN}; although every token can in principle attend to every other token, trained attention models are far from structureless. They exhibit strong positional preferences~\cite{Wu2025OnTE,liu2024lost,Guo2024SerialPE}, attention sinks~\cite{Gu2024WhenAS,Xiao2023EfficientSL}, head specialization~\cite{Olsson2022IncontextLA}, and highly non-uniform use of embedding dimensions~\cite{Jin2025MassiveVI,Barbero2024RoundAR,oka2026frequency}. These phenomena suggest that training does more than fit a predictor: it induces systematic preferences in how a flexible architecture uses its internal degrees of freedom.

Rotary Position Embeddings (RoPE) provide a particularly clean instance of this phenomenon~\cite{su2023roformerenhancedtransformerrotary}. RoPE gives a model access to a fixed grid of positional frequencies, yet trained models do not use these frequencies uniformly. Instead, query and key representations often concentrate their norm in restricted frequency bands~\cite{Barbero2024RoundAR,Jin2025MassiveVI,oka2026frequency}. This raises a fundamental question:
\begin{quote}
\vspace{-1ex}
     \hspace{4ex} \emph{What determines which RoPE frequencies a model learns to use?}
\vspace{-1ex}
\end{quote}

One possible explanation is that high RoPE frequencies encode position while low frequencies encode semantics~\cite{Barbero2024RoundAR}. This intuition is partly motivated by the fact that very low RoPE frequencies vary slowly over the training context and are therefore nearly invariant to small changes in position. However, low sensitivity to position should not be mistaken for non-positional information. A low frequency in RoPE still modulates attention scores as a function of relative distance; it simply varies more gradually across the context. Moreover, prior work shows that changing the training sequence length alone can shift RoPE frequency usage~\cite{oka2026frequency}, suggesting that frequency selection depends on characteristics of the training data rather than on a fixed semantic-versus-positional partition.

In this work, we propose a data-centered theory of RoPE frequency selection. The key object is a data-induced positional dependency profile: a relative-distance profile describing where task-relevant information lies in the sequence. Under this view, RoPE frequencies act as positional lenses. Higher frequencies provide sharper local resolution, but become ambiguous over shorter ranges because their phases wrap quickly. Lower frequencies cover broader ranges, but distinguish nearby positions more coarsely. We formalize this field-resolution tradeoff and show that, for a dependency profile of width \(W\), the optimal admissible frequency scales as \(1/W\). Thus, the frequency selected by training is matched to the positional scale of the dependencies present in the data.

This perspective also clarifies when position interpolation (PI) enables length generalization. PI rescales RoPE frequencies by \(\theta \mapsto \theta/\alpha\) to extend a model from length \(L\) to length \(\alpha L\)~\cite{chen2023extending}. We show that this operation expands the effective field of each frequency by a factor of \(\alpha\), while reducing its local resolution by the same factor. Consequently, PI should help when the test-time dependency structure is approximately a stretched version of the training-time structure. We formalize this condition through self-similarity of dependency profiles. Under this condition, PI preserves frequency utility and maps the optimal training-scale frequency to the optimal longer-context frequency. This provides a mechanism for why PI can work on natural language, whose dependencies exhibit approximate self-similarity across scales~\cite{Alabdulmohsin2024FractalPM}, and why its benefits are limited on tasks whose relevant dependencies do not scale with context length.

Taken together, our results suggest a broader view of emergent positional structure in attention-based models. RoPE does not merely provide generic position information; it exposes a spectrum of positional lenses, each with a distinct field-resolution tradeoff. Training selects among these lenses according to the dependency structures present in the data. The resulting frequency usage is therefore a learning-induced positional inductive bias: not hand-coded in the architecture, but shaped by the training distribution and encoded in the learned query-key weights.

\textbf{Our contributions are summarized as follows:}
\vspace{-1ex}
\begin{itemize}[leftmargin=2ex]
\item We show that RoPE frequency usage is a learned, input-stable, and training-data-dependent property of trained transformers, and introduce an attention score-level energy measure to quantify frequency usage.

\item We develop a data-centered theory of RoPE frequency selection. By analyzing how RoPE frequencies provide positional contrast over data-induced dependency profiles, we prove a frequency-matching principle: for a dependency profile of width \(W\), the optimal admissible frequency scales as \(\theta^\star \asymp 1/W\). This explains why broad dependency structures favor lower frequencies while local dependencies favor higher frequencies, and we validate the prediction on controlled synthetic and text-based data.

\item We connect frequency selection to length generalization by position interpolation. We show that PI expands the effective field of each frequency while reducing its resolution, and is effective when test-time dependencies are approximately stretched versions of training-time dependencies. This predicts when PI should help: it can enable length generalization when relevant dependencies scale with context length, as in natural-language generation, but can fail or offer limited benefits when such scaling structure is absent.
\end{itemize}

\begin{figure}[t]
\centering
\includegraphics[width=\linewidth]{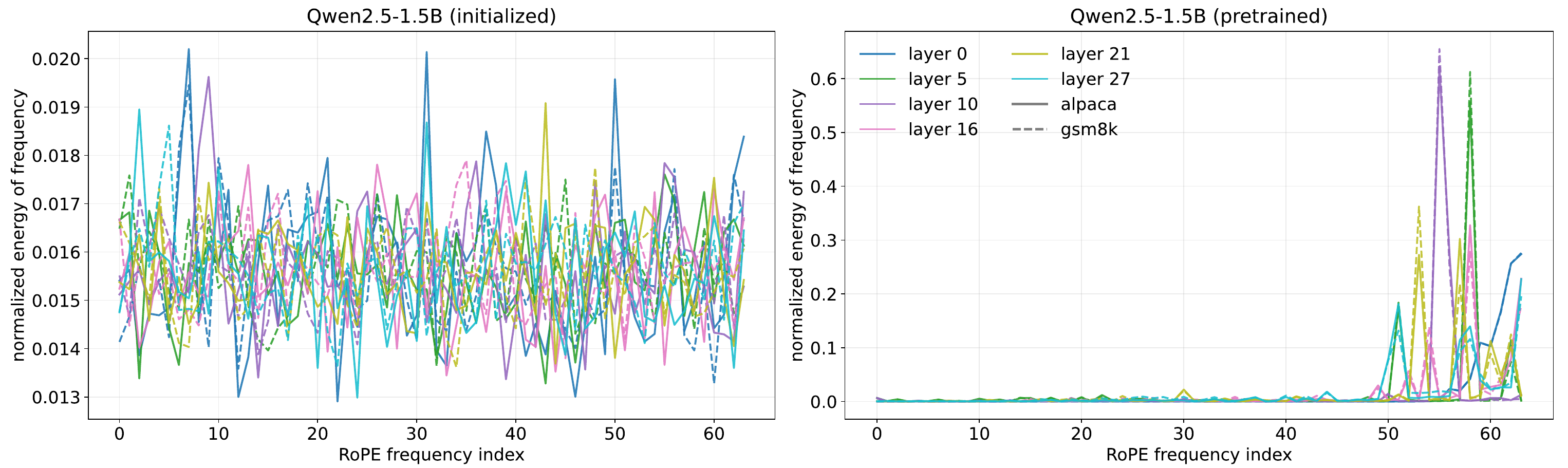}
\caption{\small Frequency energy (averaged over heads) of Qwen-2.5-1.5B model tested on alpaca and gsm8k datasets. The frequency usage concentrates on low frequencies after training and remains stable across prompts. }
\label{fig:input_stable}
\end{figure}

\vspace{-2ex}
\section{RoPE frequency usage is learned, input-stable and training data-dependent}
\vspace{-1ex}

\begin{figure}[t]
    \centering
    \includegraphics[width=\linewidth]{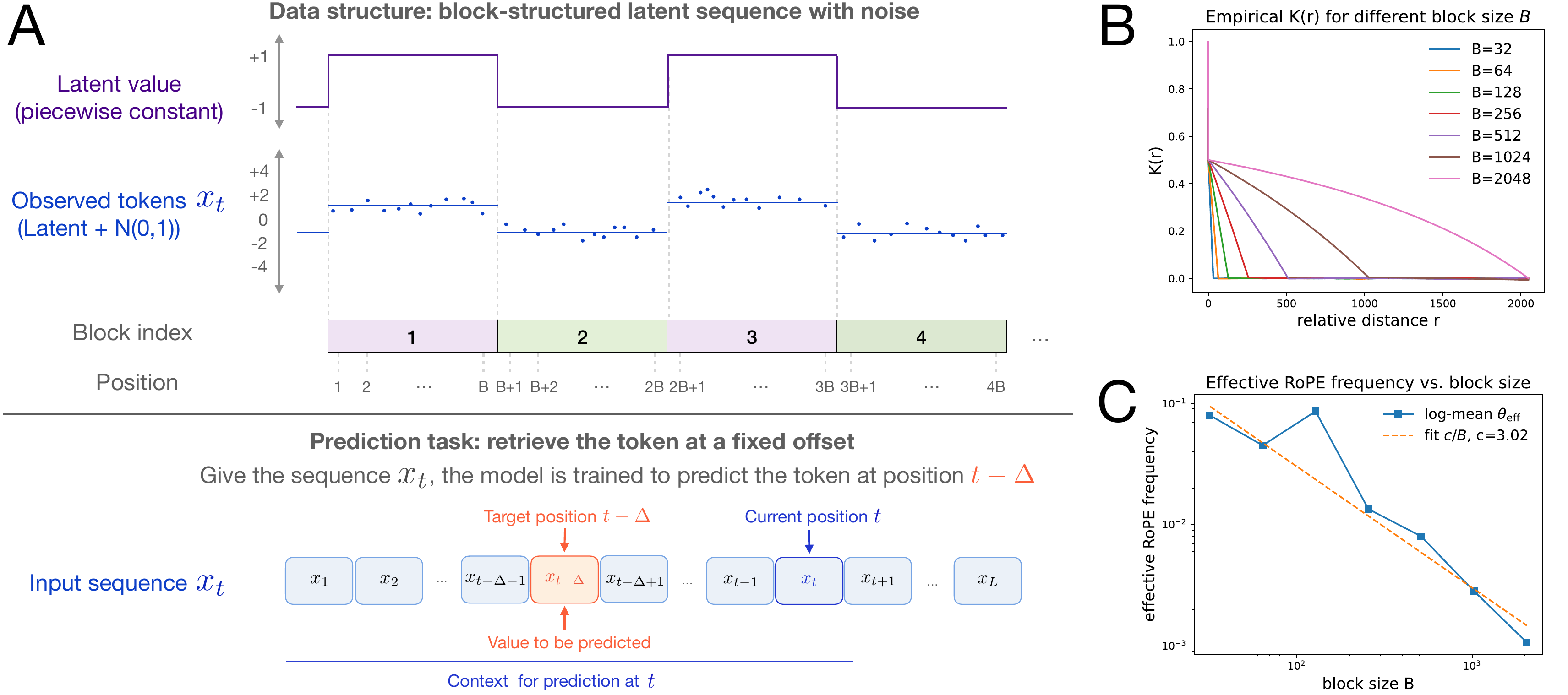}
     \caption{\small
\textbf{Block-structured data reveals frequency matching between RoPE usage and data  dependency width.}
(A) Synthetic block-drift data. Each sequence consists of latent blocks of length \(B\), where the latent value is piecewise constant and sampled uniformly from \(\{+1,-1\}\); observed tokens are obtained by adding independent unit Gaussian noise. The model is trained to predict the token at a fixed offset \(t-\Delta\) from the current position \(t\).
(B) Empirical data dependency kernels \(K(r)\), measured by autocorrelation, for different block sizes. As the dependency width equals the block length ($W=B$), increasing \(B\) broadens the dependency width $W$. 
(C) Learned RoPE frequency usage follows the predicted inverse scaling law. The log-mean effective frequency decreases as block size increases, and a fit of the form \(c/B\) gives \(c=3.02\), close to the theoretical constant \(\pi\) in~\Cref{thm:field_constrained_optimal_frequency}. This provides controlled evidence that broader dependency widths select lower RoPE frequencies.
}
    \label{fig:block_freq}
    \vspace{-4ex}
\end{figure}

RoPE equips attention with a fixed grid of positional frequencies~\cite{su2023roformerenhancedtransformerrotary}. Yet trained models use this grid in a highly non-uniform way: query and key embeddings often concentrate their norm in restricted frequency bands~\cite{Barbero2024RoundAR,Jin2025MassiveVI,oka2026frequency}. Does this concentration come from the architecture and initialization, from input-specific variation, or from structure learned during training? We first disentangle these possibilities empirically.

\vspace{-1ex}
\paragraph{Measuring RoPE frequency usage}
Prior work measures RoPE frequency usage through the norms of query and key embeddings in each RoPE dimension pair~\cite{Barbero2024RoundAR,Jin2025MassiveVI,oka2026frequency}. Since RoPE affects attention through relative rotations of queries and keys, we instead use a score-level measure that directly quantifies each frequency's contribution to the raw attention score. For each RoPE frequency \(\theta_m\), the contribution to the raw attention score between query \(q_i\) and key \(k_j\) can be written as
\vspace{-2ex}

\[
    s_{ij}^{(m)}
    =
    a_{ij,m}\cos(\theta_m(i-j))
    +
    b_{ij,m}\sin(\theta_m(i-j)),
\]
where \(a_{ij,m} = q_{i1}^{m}k_{j1}^{m} + q_{i2}^{m}k_{j2}^{m}\)  and \(b_{ij,m} = q_{i1}^{m}k_{j2}^{m} - q_{i2}^{m}k_{j1}^{m}\) depend on the query and key components associated with $\theta_m$, $q^m,k^m\in\mathbb{R}^2$. We define the energy of frequency \(\theta_m\) in an attention head as the squared amplitude of this sinusoidal contribution:
\vspace{-1ex}
\[
    E_m
    =
    \mathbb E_{i,j,x}
    \left[
        a_{ij,m}^2+b_{ij,m}^2
    \right],
\]
where the expectation is taken over inputs $x$ and token pairs $(i,j)$. We refer to the normalized collection \(\{E_m / \sum_{m'} E_{m'}\}_m\) as the \emph{RoPE frequency energy spectrum}.

\vspace{-1ex}
\paragraph{RoPE frequency usage is learned and input-stable}
We first examine whether the RoPE frequency spectrum is already structured at initialization or emerges through training. We compute \(E_m\) for Qwen2.5-1.5B~\cite{qwen2.5} on alpaca~\cite{alpaca} and gsm8k~\cite{cobbe2021gsm8k} and results are shown in~\Cref{fig:input_stable}. At random initialization, the spectrum is nearly uniform across frequencies. After training, it becomes highly concentrated, with most energy assigned to a restricted range of low frequencies. This shows that non-uniform RoPE frequency usage emerges through training. Moreover, the spectra are largely stable across inputs and even across domains, suggesting that frequency usage is encoded in the learned query-key weights rather than selected separately for each prompt.

\vspace{-2ex}
\paragraph{Training data determines what frequency RoPE uses}
\label{sec:motivation_training_data}

The next question is what determines where the learned RoPE spectrum concentrates. Prior work shows that changing the training sequence length alone can shift the frequency band used by a trained model~\citep{oka2026frequency}. This suggests that RoPE frequency usage is sensitive to the training data, but sequence length is a coarse intervention: changing it alters many properties of the data at once. For example, longer sequences can change not only the maximum available distance, but also the distribution of relevant relative distances, the amount of repeated content, and the mixture of local and long-range dependencies.

To isolate the relevant data property, we construct a controlled block-structured sequence task. Each sequence has fixed length \(L=4096\) and consists of latent blocks of length \(B\), where tokens within the same block share a common latent value up to unit Gaussian noise. The prediction task is fixed: the model must retrieve the token at offset \(\Delta=16\). We vary only the block length \(B\in \{32,64,128,256,512,1024,2048\}\), which controls how far information persists across positions in the sequence. See~\Cref{app:block} for more details about the experimental setup and ~\Cref{fig:block_freq}A for a visualization of the data and the task.

We train a two-layer single-head attention-only network,  with $\theta_m = 10000^{-2(m-1)/d}$ as in~\cite{su2023roformerenhancedtransformerrotary}, on this task and measure the RoPE frequency spectrum in the second layer. As shown in~\Cref{fig:block_freq}C, the log-mean effective frequency,
$\theta_{\text{eff}} = \exp(\sum_{m}\omega_m\log(\theta_m))$, with $\omega_m = E_m/\sum_{m'}E_{m'}$, decreases as \(B\) increases, meaning that increasing \(B\) systematically shifts the learned energy toward lower frequencies. Thus, varying a single parameter that controls the dependency scale of the data systematically moves the learned RoPE spectrum.

These observations suggest that RoPE frequency usage is a learned property of the model, shaped by the training distribution rather than selected independently for each input. The key question is therefore what property of the data makes certain frequencies useful after training. We argue that this property is the data-induced positional dependency profile: the distribution of task-relevant dependencies across relative distances. We now formalize this object and show how its characteristic width determines the useful RoPE frequency scale.

\vspace{-2ex}
\section{Theory: Positional Dependency Profiles and RoPE Frequency Selection}
\label{sec:theory}
\vspace{-1ex}

We now formalize the data-side property that controls RoPE frequency selection. The central object is a data-induced positional dependency profile, represented by a kernel \(K(r)\), which measures how strongly a prediction depends on information at relative distance \(r\). The kernel induces a dependency width \(W(K)\), the largest relative-distance scale relevant to the task. On the other hand, a RoPE frequency \(\theta\) represents relative distance $r$ through the phase \(\theta r\). A larger \(\theta\) gives finer local positional resolution but has a smaller unambiguous field before phase wrapping occurs; a smaller \(\theta\) covers a wider field but distinguishes nearby positions more coarsely.  Our analysis makes this field-resolution tradeoff precise and shows that the useful RoPE frequency scale is determined by the dependency width induced by the data.

\vspace{-2ex}
\subsection{Data-induced positional dependency kernels}
\vspace{-1ex}
Consider a sequence prediction task and a query position \(i\). Information at a previous position \(j\) may be relevant for predicting the output at \(i\). We write $r=i-j$
for the relative distance between the two positions.

\begin{definition}[Positional dependency kernel]
A positional dependency kernel is a nonnegative function or measure \(K(r) \ge 0\) over relative distances \(r \ge 0\). It measures the strength of task-relevant dependence at distance \(r\).
\end{definition}
\vspace{-1ex}

When \(K\) is integrable, we normalize it into a probability distribution over relative distances: $p_K(r) =
   K(r)/\int_0^\infty K(s)\,ds$.
 More generally, we write \(dP_K(r)\) for the normalized measure induced by \(K\), which describes how task-relevant dependencies are distributed across relative distances.

\begin{definition}[Field width of data]
If \(K\) is supported on \([0,W]\), we define its field width as $ W(K):=\sup \operatorname{supp}(K)$.
\vspace{-1ex}
\end{definition}

The compact-support definition gives an exact notion of the largest relevant relative distance and is the setting used in our main theoretical analysis. For long-tailed kernels, such as those arising in natural language, we can further define  the \(\rho\)-field width $ W_\rho(K)
    =
    \inf\left\{
    R:\int_0^R dP_K(r)\ge \rho
    \right\},\rho\in(0,1),$
 which gives an operational field width for dependency kernels that do not have compact support.

\vspace{-1ex}
\subsection{RoPE frequencies as positional lenses}
\vspace{-1ex}

We now formalize how a single RoPE frequency encodes relative position, and how its mechanism induces a tradeoff between its field width and resolution.
\vspace{-2ex}

\paragraph{Single-frequency RoPE}
Consider one two-dimensional RoPE subspace with frequency \(\theta\). Let
\(q_i^{(\theta)}, k_j^{(\theta)} \in \mathbb R^2\) denote the corresponding query and key components before applying RoPE. RoPE rotates these components by their absolute positions:
 $\widetilde q_i^{(\theta)}
    =
    R_{\theta i} q_i^{(\theta)},
    \widetilde k_j^{(\theta)}
    =
    R_{\theta j} k_j^{(\theta)}$, where \(R_\varphi\) is the two-dimensional rotation matrix with angle \(\varphi\). The contribution of this frequency to the attention score is 
    $\left\langle
    \widetilde q_i^{(\theta)},
    \widetilde k_j^{(\theta)}
    \right\rangle
    =
    \left\langle
    R_{\theta i}q_i^{(\theta)},
    R_{\theta j}k_j^{(\theta)}
    \right\rangle
    =
      \left\langle
    q_i^{(\theta)},
    R_{\theta r}k_j^{(\theta)}
    \right\rangle,$
using orthogonality of rotations. Thus, although RoPE applies rotations using absolute positions, the attention score depends on position only through the relative distance $r=i-j$.
\vspace{-2ex}

\paragraph{Relative positional information by RoPE} 
The single-frequency $\theta$ contribution to the raw attention score is $ s_{ij}^{(\theta)}
    =
    a_{ij,\theta}\cos(\theta r)
    +
    b_{ij,\theta}\sin(\theta r), r=i-j,$
where \(a_{ij,\theta} = q_{i1}^{\theta}k_{j1}^{\theta} + q_{i2}^{\theta}k_{j2}^{\theta}\)  and \(b_{ij,\theta} = q_{i1}^{\theta}k_{j2}^{\theta} - q_{i2}^{\theta}k_{j1}^{\theta}\) depend on the query and key components corresponding to $\theta$. Essentially, the frequency represents relative distance through the two-dimensional feature~\cite{Tancik2020FourierFL} $\gamma_\theta(r) = (\cos(\theta r), \sin(\theta r))$, and the attention score is a learned, content-dependent projection of this feature.

\vspace{-2ex}
\paragraph{Positional contrast} To quantify the positional separation provided by frequency \(\theta\), independent of the particular projection coefficients, we compare the feature at distance \(r\) with the reference feature at distance \(0\). We define the  \emph{positional contrast} score to be $D_\theta(r)
    =
    1-
    \left\langle
    \gamma_\theta(r),
    \gamma_\theta(0)
    \right\rangle = 1-\cos(\theta r)$,
 which measures how much frequency $\theta$ separates relative distance $r$ from the reference distance $0$. Larger values of \(D_\theta(r)\) means stronger positional contrast.

\vspace{-2ex}
\paragraph{Resolution} This contrast induces a natural notion of \emph{resolution}. For a threshold \(\tau\in(0,2)\), define the resolution length $\Delta_\tau(\theta)
    =
    \inf\left\{
    r>0:
    1-\cos(\theta r)\ge \tau
    \right\} =
    \arccos(1-\tau)/\theta$. This quantity measures the smallest distance at which positions become distinguishable at level \(\tau\). Larger frequencies yield smaller resolution length and therefore finer positional resolution.

\vspace{-2ex}
\paragraph{Field} While the relative-position feature \(\gamma_\theta(r)\) is periodic with period \(2\pi/\theta\), the position contrast \(D_\theta(r)\) is monotone only over a half period.  We therefore define the \emph{field of frequency} \(\theta\) as $F(\theta)=\pi/\theta$.
 Within \([0,F(\theta)]\), larger relative distances correspond to larger positional contrast. Beyond this range, the contrast begins to decrease, so different distances can produce the same contrast value and the frequency no longer provides an unambiguous ordering of distances from the reference.

\vspace{-1ex}
These definitions give the field-resolution tradeoff: larger \(\theta\) improves local resolution but shrinks the unambiguous field; smaller \(\theta\) covers a wider field but distinguish nearby positions more coarsely.

\vspace{-2ex}
\paragraph{Admissible frequencies}
A dependency kernel \(K(r)\) induces a dependency width \(W\). For a frequency to be useful over the full task-relevant region, its field must cover this width.

\begin{definition}[Field-admissible frequency]
A frequency \(\theta\) is admissible for a dependency kernel of width \(W\) if $F(\theta)\ge W$, or equivalently, $\theta W \le \pi$.
\vspace{-1ex}
\end{definition}

We denote the admissible set by $\mathcal A_W
    =
    \left\{
    \theta>0:\theta W\le \pi
    \right\}$. Admissibility formalizes a coverage requirement: the frequency must remain unambiguous over the relevant dependency width. Among admissible frequencies, we then ask which one provides the strongest positional contrast.

\vspace{-2ex}
\subsection{Field-constrained optimal frequency}
\vspace{-1ex}
Let \(K\) be a positional dependency kernel supported on \([0,W]\), and let \(P_K\) be its normalized measure. Define the average positional contrast of a frequency \(\theta\) against \(K\) by
\vspace{-1ex}
\[
    U(\theta;K)
    =
    \int_0^W
    \bigl(1-\cos(\theta r)\bigr)
    \,dP_K(r).
\]
This utility measures how much positional contrast the frequency \(\theta\) provides over task-relevant distances specified by \(K\). Larger values of \(U(\theta;K)\) indicate that the frequency better separates positions that matter for the task.

\begin{theorem}[Field-constrained optimal frequency]
\label{thm:field_constrained_optimal_frequency}
Let \(W>0\). Let \(P_K\) be a probability measure on \([0,W]\) satisfying $  P_K((0,W])>0$.
Then \(U(\theta;K)\) is strictly increasing on the admissible set $\mathcal A_W$. Consequently, the unique maximizer is $\theta^\star=  \underset{\theta\in\mathcal A_W}{\argmax}\,U(\theta, K) = \pi/W$.
\end{theorem}
\vspace{-1ex}

The theorem gives a precise frequency-matching principle. Among all frequencies whose field covers the data-induced dependency region, the most useful one is the highest admissible frequency. Therefore, when the dependency has width \(W\), the preferred frequency scales as: $\theta^\star \asymp 1/W$. Broader dependency widths favor lower frequencies, while narrower dependency widths favor higher frequencies. This is the theoretical basis for the frequency-matching phenomena observed in~\Cref{sec:motivation_training_data} (see \Cref{fig:block_freq}).

\vspace{-2ex}
\paragraph{Natural language as a mixture of positional dependencies}
\Cref{thm:field_constrained_optimal_frequency} analyzes the single-profile case, where the dependency kernel has one characteristic width \(W\). Natural language contains a mixture of positional dependency scales, from local collocations to phrase-, sentence-, and discourse-level structure. Thus, the theorem should not be read as predicting a single optimal frequency for language. Instead, it gives the scale-wise principle: wider dependencies favor lower admissible frequencies, while narrower dependencies favor higher ones. Since RoPE provides a fixed discrete grid of frequencies, training must allocate frequency energy across the available grid according to the dependency scales present in the data. This provides a principled lens for interpreting the mid-low frequency usage band observed in real language models~\cite{Barbero2024RoundAR,Jin2025MassiveVI,oka2026frequency}.





\vspace{-2ex}
\section{Position interpolation as field dilation}
\label{sec:pi_theory}
\vspace{-2ex}
RoPE supports inference-time extrapolation to longer sequences via position interpolation (PI)~\cite{chen2023extending}. For a model trained on sequences of length \(L\), PI applies the model to length \(\alpha L\), \(\alpha>1\), by rescaling each RoPE frequency as \(\theta \mapsto \theta/\alpha\). This simple operation underlies many RoPE-based length generalization methods~\cite{chen2023extending, bloc97_2023_ntk_by_parts, bloc97_2023_ntk_rope,Peng2023YaRNEC, emozilla_2023_dynamic_rope}. What does this operation do to the positional information learned during training so that it achieves "length generalization"? We now interpret PI through the field-resolution framework developed in~\Cref{sec:theory}: PI dilates the field of each frequency while reducing its resolution.

\vspace{-2ex}
\subsection{Resolution-field tradeoff under position interpolation}
\vspace{-1ex}

Recall that a RoPE frequency \(\theta\) represents relative distance through
 $\gamma_\theta(r) = (\cos(\theta r), \sin(\theta r))$. Under PI, the same distance \(r\) is represented using frequency \(\theta/\alpha\): $ \gamma_{\theta/\alpha}(r)
    =
    (\cos((\theta/\alpha)r),
    \sin((\theta/\alpha)r)) 
    =
    \gamma_\theta\left(r/\alpha\right)$. Thus a distance \(r\) in the longer sequence is represented exactly as distance \(r/\alpha\) would have been represented at the original frequency. This is why PI can be viewed as interpolation rather than direct extrapolation in absolute position.

From the perspective of the field–resolution tradeoff, this transformation has two opposing effects. First, PI expands the field of each frequency. Since the field width of a frequency $\theta$ is $F(\theta)= \pi/\theta$, after rescaling, we obtain 
    $F\left(\theta/\alpha\right)
    =
    \pi/(\theta/\alpha)
    =
    \alpha F(\theta)$, meaning the effective field increases by a factor of $\alpha$. At the same time, this expansion comes at the cost of reduced positional resolution. Recall that the resolution length is $\Delta_\tau(\theta) = \arccos(1-\tau)/\theta$. Under PI, this becomes $\Delta_\tau\left(\theta/\alpha\right)
    =
    \arccos(1-\tau)/(\theta/\alpha)
    =
    \alpha\Delta_\tau(\theta)$. Thus, the resolution length also increases by a factor of $\alpha$, meaning that positional distinctions become coarser at the same rate as the field expands.

\vspace{-2ex}
\subsection{Self-similar tasks and extrapolation by interpolation}
\vspace{-1ex}

The field-resolution tradeoff shows what PI does mechanically: it expands the field size of each RoPE frequency while reducing its resolution. 
The remaining question is when this tradeoff supports length generalization. Since useful RoPE frequencies are matched to the scale of the data-induced dependency profile, rescaling frequencies at inference time should help when the dependency structure in the longer context is itself a rescaled version of the structure seen during training. We formalize this condition as self-similarity.

\begin{definition}[Self-similarity]
    Let \(P_W\) be a normalized positional dependency measure with width \(W\). We say the task family is self-similar if for every \(\alpha>1\), $P_{\alpha W}
    =
    (S_\alpha)_\#P_W, 
    S_\alpha(r)=\alpha r,$
where \((S_\alpha)_\#P_W\) is the pushforward of \(P_W\) under the dilation map \(S_\alpha\). 
\end{definition}
\vspace{-1ex}

Equivalently, for any measurable function \(f\), $\int f(r)\,dP_{\alpha W}(r)
    =
    \int f(\alpha s)\,dP_W(s)$. In words, self-similarity means that the long-context dependency structure a scaled (dilated) version of the short-context one along the position axis.
We now show that, under this self-similarity condition, position interpolation preserves the positional utility of each frequency component. We write \(U(\theta,W)\) for the utility \(U(\theta;P_W)\) associated with the dependency kernel at width \(W\).

\begin{theorem}
If $P_{\alpha W}=(S_\alpha)_\#P_W$,
then $U\left(\theta/\alpha,\alpha W\right)
    =
    U(\theta, W).$
\label{thm:PI}
\end{theorem}
\vspace{-1ex}

Thus, when the dependencies dilates with context length, PI exactly preserves positional contrast utility. The frequency component useful at width \(W\) remains equally useful at width \(\alpha W\) after rescaling \(\theta\) to \(\theta/\alpha\). PI therefore does not create a new positional mechanism for longer contexts; it reuses the training-scale mechanism under a rescaled distance coordinate.

\begin{corollary}[Optimal frequency scales under self-similarity]
Assume the conditions of~\Cref{thm:PI}.
If $\theta_W^\star
    =
    \argmax_{\theta\in\mathcal{A}_W} U(\theta, W)$, then
 $\theta_W^\star/\alpha
   = \argmax_{\theta'\in\mathcal{A}_{\alpha W}} U(\theta',\alpha W)$.
\label{cor:PI}
\end{corollary}
\vspace{-1ex}

The corollary shows that, under self-similarity, the optimal frequency at the longer scale is exactly the PI-scaled version of the optimal frequency at the training scale. This gives a precise condition under which frequency scaling should support length generalization.

We note that when this self-similarity condition fails, PI need not help. If the longer-context dependency structure is not a dilation of the training-time one, then frequency rescaling may move attention away from the relevant relative distances or blur the resolution required by the task. In~\Cref{sec:PI}, we present examples illustrating both regimes: PI succeeds when dependencies scale with context length, and fails when they are not. We then show that the approximate self-similarity of natural-language dependencies~\cite{Alabdulmohsin2024FractalPM} provides a data-side foundation for why PI can support length generalization in language models.

\vspace{-2ex}
\section{Experiments}
\vspace{-2ex}

\subsection{Dependency width predicts RoPE frequency usage}
\label{sec:freq_exp}
\vspace{-1ex}

We first test the scale-matching prediction of our theory: broader dependency width should induce lower-frequency RoPE usage. We begin with the block-structured data from~\Cref{sec:motivation_training_data}. In this setting, the block length \(B\) directly controls the range over which information persists across positions. We measure the autocorrelation of the generated sequences and find that the dependency width satisfying $W=B$, as shown in~\Cref{fig:block_freq}B. We then fit the measured effective frequency to the predicted inverse law $\theta_{\mathrm{eff}} = c/W$, 
and obtain \(c=3.02\), close to the theoretical constant \(\pi\) in~\Cref{thm:field_constrained_optimal_frequency}. This provides a direct controlled validation of the field-frequency matching law.

We next ask whether the same principle appears in text-like pretraining data. Since natural language does not allow direct control of dependency width, we use the iGSM dataset~\cite{YXLA2024-gsm1}, which generates grade-school math problems with controllable reasoning length. We vary the number of operations, \(\text{ops}\in\{2,5,10\}\). Larger \(\text{ops}\) produces longer questions and broader dependencies, which we verify by estimating the mutual-information profile as shown in~\Cref{fig:iGSM_field_frequency}A. See \Cref{app:iGSM} for more details.

We then train separate 12-layer GPT models on the three iGSM datasets using the nanochat framework~\cite{nanochat} and measure their learned RoPE frequency usage. For each layer \(l\) and head \(h\), we summarize the spectrum by the log-effective frequency $\log \theta_{\mathrm{eff}}
    =
    \sum_m \omega_m \log \theta_m,$
where \(E_m\) is the RoPE energy assigned to the \(m\)-th frequency,  $\omega_m
    =
    E_m/\sum_n E_n$, and $\theta_m = 100000^{-2(m-1)/d}$. A smaller \(\log \theta_{\mathrm{eff}}\) indicates greater use of lower frequencies. As shown in~\Cref{fig:iGSM_field_frequency}B, increasing the dependency width by increasing number of ops shifts RoPE usage toward lower frequencies across layers, consistent with the scale-matching prediction given by our theory.

\begin{figure}[t]
\centering
\includegraphics[width=\linewidth]{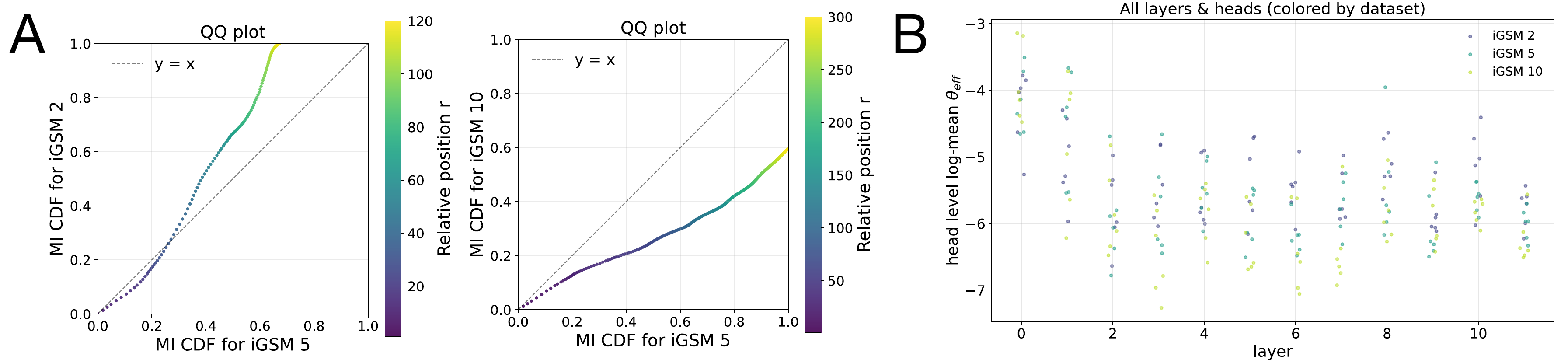}
\caption{\small
\textbf{Data dependency width predicts RoPE frequency usage.}
(A) Mutual-information CDFs over relative positions show that increasing the number of operations broadens the dependency width: iGSM-2 concentrates information at shorter distances, while iGSM-5 and iGSM-10 shifts dependency mass to longer distances.
(B) Head-level log-mean effective frequencies for 12-layer transformer models trained separately on iGSM-2, iGSM-5, and iGSM-10. Broader dependencies induce lower effective RoPE frequency usage, consistent with our frequency scale-matching theory (\Cref{sec:theory}).
}
\label{fig:iGSM_field_frequency}
\vspace{-3ex}
\end{figure}

\vspace{-2ex}
\subsection{When Does PI Enable Inference-Time Length 
Generalization?}
\label{sec:PI}
\vspace{-1ex}
\Cref{sec:pi_theory} predicts that PI should help when test-time dependencies are approximately stretched versions of training-time dependencies. We now test this prediction and examine what happens when this self-similar structure is absent. We consider three cases: a synthetic retrieval task that isolates the success and failure modes, natural language with approximately self-similar dependencies, and arithmetic tasks with non-self-similar dependencies.

\begin{figure}
    \centering
    \includegraphics[width=\linewidth]{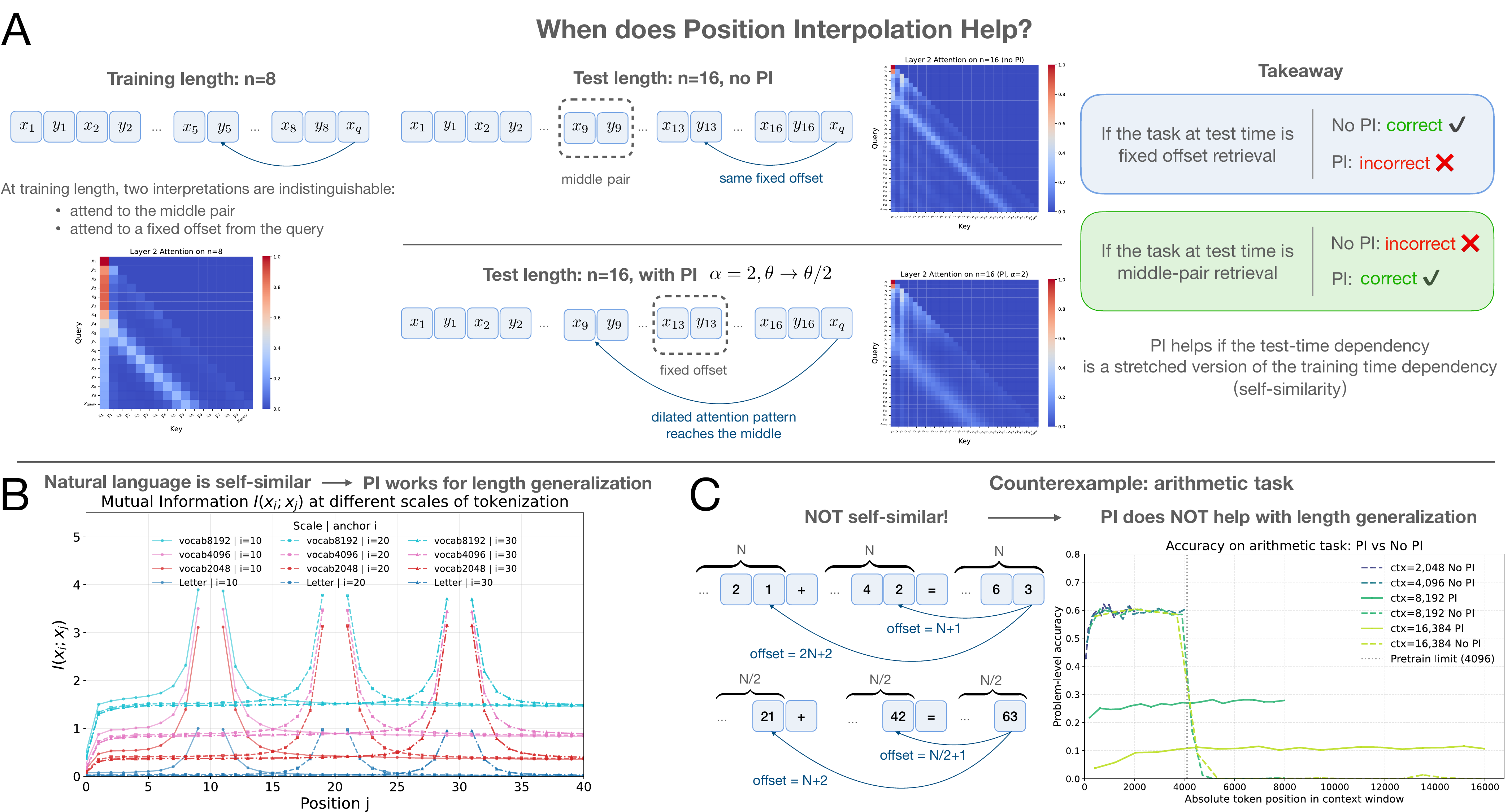}
    \caption{\small
\textbf{PI is effective when dependencies scale with context length, and can fail when they do not.}
(A) At training length \(n=8\), attending to the middle pair and attending to a fixed offset are indistinguishable. At test length \(n=16\), No PI preserves the fixed-offset pattern, whereas PI with \(\alpha=2\) dilates the attention pattern toward the middle. Hence PI helps for middle-pair retrieval but hurts fixed-offset retrieval.
(B) Natural language shows approximately self-similar dependency structure across tokenization scales, as reflected by similar mutual-information profiles, providing a basis for PI-based length generalization.
(C) Arithmetic tasks are not self-similar: they require precise local alignment rather than stretched dependencies. PI therefore trades resolution for field size and fails to provide true length generalization.
}
    \label{fig:pi_meta}
    \vspace{-3.5ex}
\end{figure}

\vspace{-2ex}
\paragraph{PI is effective when dependencies scale with context length} We first consider a synthetic retrieval task that admits two indistinguishable strategies at the training length. Following~\citet{Reddy2023TheMB, Wu2025OnTE}, we consider a sequence of the form \( x_1, y_1, \dots, x_n, y_n, x_{\text{query}} \), and the model is trained to predict the label $y_{\text{query}}$
of a target $x_{\text{query}}$ using the cross-entropy loss. During training, we set \(n=8\) and $x_{\text{query}}$ is constructed to match the label of $x_{n/2+1}$ in the sequence. As in the block data experiment, we train a two-layer attention-only single-head network and examine the attention patterns in the second layer. See~\Cref{app:key-value} for more experimental details. The visualization of the task and the resulting attention patterns are shown in~\Cref{fig:pi_meta}A. 

At the training length ($n=8$), there are two possible ways to interpret the learned behavior of the model: (1) the model may learn to attend to the middle pair of the sequence, namely $(x_{n/2+1}, y_{n/2+1})$, or (2) the model may instead learn attend to a fixed offset from the query. These two interpretations are indistinguishable at the training length, since the middle position coincides with a fixed offset from the query. We then evaluate the trained model on sequences of doubled length ($n = 16$). Without PI, the model attends to tokens at the same fixed offset relative to the query, indicating that it has learned an absolute-offset strategy (2). However, with PI ($\alpha = 2$), the attention now attends toward the middle region of the longer sequence. 

This illustrates the central prediction of our theory: PI does not generically improve length generalization. it helps when the relevant dependency at test time is a scaled version of the dependency learned during training. Depending on the notion of "correctness" at longer sequence lengths, If the test task still requires retrieving the token at the same fixed offset from the query, then PI hurts, since it moves attention away from the fixed-offset target. If instead the test task requires retrieving the corresponding relative position, such as the middle element of the longer sequence, then PI helps.

\vspace{-2ex}
\paragraph{Natural language is approximately self-similar}

Our theory predicts that PI should be effective when the dependency structure of the data is approximately self-similar. Since PI has been shown to extend context length with limited perplexity degradation on natural language~\cite{chen2023extending}, a natural question is whether this success reflects an approximate self-similarity property of language rather than a generic benefit of frequency scaling. To test this, we measure the average mutual information \(I(x_i;x_j)\) between tokens at different relative distances on the Nemotron-ClimbMix pretraining dataset~\cite{diao2025climb}. 

To approximate a positional dilation of natural language, we vary the tokenization granularity and examine whether dependency profiles remain stable across scales. We train Byte-Pair Encoding (BPE) tokenizers with different vocabulary sizes: larger vocabularies produce coarser tokenizations, with more information per token and fewer tokens per text span. We then compute mutual information on the resulting tokenized sequences. For more details see \Cref{app:MI}. 

As shown in~\Cref{fig:pi_meta}B, the mutual-information profiles remain similar across tokenization scales. This suggests that natural-language dependency structure is approximately self-similar, consistent with prior work~\cite{Alabdulmohsin2024FractalPM}. This provides a data-side explanation for why PI can support length generalization in language models: when longer-context dependencies are stretched versions of shorter-context dependencies, scaling down RoPE frequencies lets the model reuse the positional patterns learned during training.

\vspace{-2ex}
\paragraph{Natural language has long-tailed dependencies}
Another observation from~\Cref{fig:pi_meta}B is that, although the mutual-information profiles decay rapidly at short distances, they remain nonzero over long ranges. This is also consistent with prior evidence that natural language exhibits long-range dependencies~\cite{Alabdulmohsin2024FractalPM, Li1989MutualIF}. In our framework, this heavy tail implies that increasing the training sequence length \(L\) also increases the effective dependency width \(W\), because additional long-range dependency mass becomes visible at larger distances. This perspective also helps explain two empirical observations: training on longer sequences shifts RoPE usage toward lower frequencies by exposing broader dependencies~\cite{oka2026frequency}, and lower frequencies are essential for longer-context tasks because they provide larger unambiguous fields~\cite{chen2023extending, Peng2023YaRNEC,Men2024BaseOR}.

\vspace{-2ex}
\paragraph{PI provides substantially weaker length generalization on non-self-similar arithmetic tasks}
We next test PI in a setting where the dependency structure is not self-similar and the task requires fine-grained positional resolution. We evaluate PI with Llama-2-7B~\cite{touvron2023llama}, which has a context length of \(4096\), on an arithmetic dataset in which each sequence is formed by concatenating randomly sampled 1--5 digit arithmetic problems. We measure exact answer-token accuracy at different context lengths using teacher-forced decoding. Detailed setups are provided in \Cref{app:PI_arithmetic}. 

\vspace{-1ex}
The results are shown in~\Cref{fig:pi_meta}C. Without PI, the model performs well within the training context length but fails beyond it. With PI, nontrivial performance extends to longer contexts, but accuracy drops across the range. This matches our theory: PI expands the effective field by sacrificing positional resolution, which is useful only when the relevant dependencies scale with context length. Arithmetic lacks this self-similar structure. 
Longer contexts add more problems, but do not dilate the dependency structure of each problem, and each problem still requires precise positional alignment. PI therefore broadens the field without preserving the fine-scale resolution needed for accurate arithmetic.

\vspace{-2ex}
\paragraph{Perplexity can overstate long-context ability}
In the arithmetic task, PI can yield lower perplexity than the no-PI baseline even when exact answer accuracy is low (see~\Cref{app:exp_add_ppl}). Inspecting answer tokens shows why: at a 16k context length, the answer-token perplexity is approximately \(16\), worse than the perplexity \(10\) of a uniform guess over digits \(0\)--\(9\). Furthermore, unlike in natural-language tasks where PI preserves perplexity with minimal degradation beyond the training context length~\cite{chen2023extending}, PI in this arithmetic setting still produces substantially higher perplexity on long contexts than the original model achieves within its training context. Thus, lower overall perplexity does not necessarily indicate long-context ability. This supports the view that perplexity can be an unreliable proxy for long-context generalization, especially when the task requires precise positional resolution~\cite{fang2024wrong}.

\vspace{-2ex}
\section{Discussion}
\vspace{-2ex}

In this paper, we study RoPE frequency usage in transformers through a data-centered lens, developing a theoretical framework that connects learned positional spectra to the dependency structure of the training data. Our analysis reveals two key findings. First, RoPE frequencies act as positional lenses, trading off local resolution against the field over which they remain unambiguous; consequently, broader dependencies favor lower frequencies, while local dependency structures favor higher frequencies. Second, the same field-resolution tradeoff explains when PI supports length generalization: scaling frequencies down expands the effective field but blurs local resolution, helping when test-time dependencies are approximate dilations of those seen during training. More broadly, our results suggest that positional encodings should be studied together with the data distributions on which they are trained. The same architecture can learn different positional resolutions depending on which dependencies the data makes useful. This perspective points toward data- and model-guided context extension: instead of applying fixed frequency-scaling rules, future methods could use measured dependency widths and learned frequency spectra to choose frequency grids, scaling schedules, and evaluations that better match the target domain.


\bibliography{references}
\bibliographystyle{plainnat}


\appendix

\section{Related Work}
\paragraph{RoPE variants and analyses of RoPE}
RoPE encodes relative position by rotating query and key features with a fixed grid of sinusoidal frequencies~\citep{su2023roformerenhancedtransformerrotary}, connecting it naturally to Fourier feature representations~\citep{Tancik2020FourierFL}. This frequency view has motivated several methods for long-context extension. Position Interpolation rescales RoPE frequencies to extend pretrained models to longer contexts~\citep{chen2023extending}, while NTK-aware scaling, dynamic scaling, and YaRN modify the scaling rule across frequencies to improve extrapolation~\citep{bloc97_2023_ntk_rope,emozilla_2023_dynamic_rope,Peng2023YaRNEC}. Recent work further analyzes RoPE from a signal-processing perspective: Fourier Position Embedding interprets RoPE-based attention through a non-uniform discrete Fourier transform and proposes a Fourier-based alternative for length generalization~\citep{Hua2025FourierPE}, while \citet{Men2024BaseOR} study how the RoPE base constrains attainable context length.

Several works study how different RoPE dimensions and frequencies are used in trained models. \citet{Barbero2024RoundAR} identify distinct frequency bands and propose p-RoPE, which removes the lowest-frequency rotations. \citet{oka2026frequency} show that frequency usage depends on the RoPE base and training context length. \citet{Jin2025MassiveVI} find that large-magnitude activations concentrate in specific Q/K dimensions, play an important role in contextual knowledge understanding, and are closely tied to RoPE. These works reveal important structure in RoPE frequency behavior and Q/K representations, but they primarily focus on architectural choices or model-internal properties. Our work instead asks what makes a frequency useful in the first place. This question is natural for attention: since attention architectures are designed to learn relevant interactions from data, understanding their positional behavior requires understanding what dependency structure the data makes useful. We show that RoPE frequency usage is shaped by the relative-distance dependency structure of the training distribution. This data-centered view explains why learned spectra change across training distributions, why broader dependency widths favor lower frequencies, and why position interpolation supports length generalization when longer-context dependencies resemble scaled versions of those seen during training.

\paragraph{Learning positional information in attention}
A growing body of work shows that attention models develop systematic positional behavior. In retrieval and ranking settings, model performance can depend strongly on where relevant information appears in the context, leading to phenomena such as lost-in-the-middle behavior~\citep{liu2024lost,Hou2023LargeLM,Zheng2023JudgingLW}. Large language models also exhibit serial position effects across tasks~\citep{Guo2024SerialPE}, and in-context learning can vary substantially with the order of demonstrations~\citep{Lu2021FantasticallyOP,Min2022RethinkingTR,zhao2021calibrateuseimprovingfewshot,Fan2025InvICL}. Other work identifies architectural sources of positional bias: \citet{Xiao2023EfficientSL} and \citet{Gu2024WhenAS} study attention sinks, where models attend disproportionately to early tokens; \citet{Wang2024EliminatingPB} observe that causal masking and RoPE can introduce position dependence; and \citet{Wu2025OnTE} show the multi-layer effects of masks and positional encodings~\citep{Wu2025OnTE}.

Several works further study how masks and positional encodings affect length generalization and positional processing. \citet{Yun2020OnCA} analyze the approximation power of transformers under different masking schemes, while \citet{Wu2024OnTR} study how attention masks affect rank collapse. For positional encodings, \citet{kazemnejad2023impact} show that many commonly used positional encodings fail to generalize to longer contexts, while NoPE can generalize better in some cases. Mitigation strategies, including alternative positional encodings~\citep{kazemnejad2023impact,Zhang2024FoundIT}, masking modifications~\citep{Wang2024EliminatingPB,Fan2025InvICL}, and bootstrapping methods~\citep{Hou2023LargeLM}, have been proposed, but are often task-specific and empirically driven.

Our work is complementary to this line of research. Rather than focusing only on architectural sources of positional information, we study how positional structure is learned from the data. Since attention is designed to learn relevant interactions from the training distribution, its positional behavior should depend not only on the positional encoding or mask, but also on the positional structure that the data makes useful. We formalize this idea for RoPE by showing that learned frequency usage is shaped by the data-induced dependency profile, and that this data-side structure also determines when position interpolation supports length generalization.

\section{Proof of \Cref{thm:field_constrained_optimal_frequency}}
\label{proof:field_constrained_optimal_frequency}
\begin{proof}
First, \(U(\theta;K)\) is well-defined. For every \(r\) and \(\theta\),
\[
    -1\le \cos(\theta r)\le 1,
\]
so
\[
    0\le 1-\cos(\theta r)\le 2.
\]
Since \(P_K\) is a probability measure,
\[
    0\le U(\theta;K)\le 2.
\]

For fixed \(r\), the function
\[
    \theta\mapsto 1-\cos(\theta r)
\]
is differentiable, with derivative
\[
    \frac{\partial}{\partial \theta}
    \bigl(1-\cos(\theta r)\bigr)
    =
    r\sin(\theta r).
\]
For every \(r\in[0,W]\),
\[
    |r\sin(\theta r)|\le r\le W.
\]
The derivative is therefore dominated by the \(P_K\)-integrable function \(W\). By dominated convergence, we may differentiate under the integral sign:
\[
    \frac{d}{d\theta}U(\theta;K)
    =
    \int_0^W r\sin(\theta r)\,dP_K(r).
\]

Now fix
\[
    \theta\in\left(0,\frac{\pi}{W}\right).
\]
For every \(r\in[0,W]\),
\[
    0\le r\le W.
\]
Multiplying by \(\theta>0\), we obtain
\[
    0\le \theta r\le \theta W.
\]
Since \(\theta<\pi/W\),
\[
    \theta W<\pi.
\]
Therefore
\[
    0\le \theta r<\pi
\]
for every \(r\in[0,W]\). On \([0,\pi]\), the sine function is nonnegative, and on \((0,\pi)\) it is strictly positive. Hence
\[
    \sin(\theta r)\ge 0
\]
for all \(r\in[0,W]\), and
\[
    \sin(\theta r)>0
\]
for all \(r>0\).

Thus the integrand \(r\sin(\theta r)\) is nonnegative on \([0,W]\) and strictly positive on \((0,W]\). Since
\[
    P_K((0,W])>0,
\]
the integral is strictly positive:
\[
    \frac{d}{d\theta}U(\theta;K)
    =
    \int_0^W r\sin(\theta r)\,dP_K(r)
    >0.
\]
Therefore \(U(\theta;K)\) is strictly increasing on
\[
    \left(0,\frac{\pi}{W}\right).
\]

Since \(U\) is continuous on the closed interval
\[
    \left[0,\frac{\pi}{W}\right],
\]
and strictly increasing on its interior, its unique maximizer over the admissible interval is the right endpoint:
\[
    \theta^\star=\frac{\pi}{W}.
\]
\end{proof}

\section{Proof of \Cref{thm:PI}}
\label{proof:PI}

\begin{proof}
By definition,
\[
    U\left(\frac{\theta}{\alpha}, \alpha W\right)
    =
    \int_0 ^{\alpha W}
    \left(
    1-\cos\left(\frac{\theta}{\alpha}r\right)
    \right)
    \,dP_{\alpha W}(r).
\]
Since
\[
    P_{\alpha W}=(S_\alpha)_\#P_W,
\]
for any measurable function \(f\),
\[
    \int f(r)\,dP_{\alpha W}(r)
    =
    \int f(\alpha s)\,dP_W(s).
\]
Choose
\[
    f(r)
    =
    1-\cos\left(\frac{\theta}{\alpha}r\right).
\]
Then
\[
\begin{aligned}
    U\left(\frac{\theta}{\alpha},\alpha W\right)
    &=
    \int
    \left(
    1-\cos\left(\frac{\theta}{\alpha}\alpha s\right)
    \right)
    \,dP_W(s)\\
    &=
    \int
    \left(
    1-\cos(\theta s)
    \right)
    \,dP_W(s)\\
    &=
    U(\theta,W).
\end{aligned}
\]
\end{proof}

\section{Experiments}

\subsection{Overview of block-structured retrieval data and experimental setup}
\label{app:block}

To isolate how the positional scale of the data affects RoPE frequency usage, we construct a controlled block-structured sequence task. We denote the sequence length by \(T=4096\). Each sequence is partitioned into contiguous latent blocks of length \(B\), where
\[
B \in \{32,64,128,256,512,1024,2048\}.
\]
Thus, each sequence contains \(T/B\) latent blocks. For each block \(j\), we sample a latent value \(s_j\) uniformly from \(\{+1,-1\}\). Each observed token \(x_t\) is generated by adding independent unit Gaussian noise to the latent value of its block:
\[
x_t = s_{b(t)} + \eta_t,
\qquad
\eta_t \sim \mathcal N(0,1),
\]
where \(b(t)\) denotes the block index containing position \(t\). Therefore, tokens within the same block share a common latent value but differ due to independent noise.

The prediction task is fixed across all experiments. Given the context up to the current position, the model is trained to retrieve the token at a fixed offset \(\Delta=16\). That is, for each valid position \(t>\Delta\), the target is
\[
x_{t-\Delta}.
\]
We use the same offset \(\Delta=16\) for all choices of \(B\). Hence, the only data parameter varied across experiments is the block length \(B\), which controls the range over which information persists across positions in the sequence.

We train a two-layer attention-only network on this task and measure the RoPE frequency spectrum in the second layer. For RoPE, we use the same frequency grid as in the main experiments:
\[
\theta_m = 10000^{-2(m-1)/d},
\]
as in~\citet{su2023roformerenhancedtransformerrotary}.

To summarize frequency usage, we compute the log-mean effective frequency
\[
\theta_{\mathrm{eff}}
=
\exp\left(
\sum_m \omega_m \log(\theta_m)
\right),
\qquad
\omega_m
=
\frac{E_m}{\sum_{m'} E_{m'}},
\]
where \(E_m\) denotes the measured RoPE energy at frequency \(\theta_m\). This quantity provides a single-number summary of where the learned frequency spectrum concentrates on the logarithmic RoPE frequency grid.

\paragraph{Training details} All models were trained on a Tesla V100 GPU. In all experiments, we used the AdamW optimizer~\cite{Loshchilov2017DecoupledWD} with learning rate \(10^{-3}\), weight decay \(5\times10^{-4}\), batch size \(32\), and trained for \(200\) iterations.

\subsection{Overview of key--value retrieval data and experimental setup}
\label{app:key-value}

Following~\citet{Reddy2023TheMB}, we consider a key--value retrieval task in which the model is trained to predict the label \(y_{\mathrm{query}}\) of a target item \(x_{\mathrm{query}}\) from an alternating sequence of \(n\) items and \(n\) labels:
\[
x_1, y_1, \ldots, x_n, y_n, x_{\mathrm{query}} .
\]
The model is trained with cross-entropy loss. Each sequence is embedded in \(d\) dimensions. Each item \(x_i\) is sampled from a Gaussian mixture model with \(K\) classes, and its associated label \(y_i\) is assigned before training from a label vocabulary of size \(L\), where \(L \leq K\). The burstiness parameter \(B\) denotes the number of occurrences of items from a given class within an input sequence. Importantly, each sequence contains at least one context item from the same class as the query.

The input distribution is controlled by several parameters. In addition to the parameters described in the main text, each Gaussian mixture class \(k\) is defined by a \(d\)-dimensional vector \(\mu_k\), whose entries are sampled i.i.d. from a normal distribution with mean zero and variance \(1/d\). Given class \(k\), the item embedding is generated as
\[
x_i = \frac{\mu_k + \epsilon \eta}{\sqrt{1+\epsilon^2}},
\]
where \(\eta\) is sampled from the same distribution as the class means, and \(\epsilon\) controls the within-class variability. Each class is assigned to one of \(L\) labels, with the label embeddings sampled before training from the same distribution as the class means.

\citet{Reddy2023TheMB} showed that different choices of the data-generating parameters lead to distinct learning regimes. To ensure that the model has strong information retrieval ability, we use the configuration identified by~\citet{Reddy2023TheMB} as corresponding to difficult in-weight learning and easy in-context learning. Specifically, we set \(\gamma=0.75\), \(K=2048\), \(L=32\), and \(B=4\).

For RoPE, we set
\[
\theta_i = 10000^{-2(i-1)/d},
\]
as in~\citet{su2023roformerenhancedtransformerrotary}.

\paragraph{Training details}
All models were trained on a Tesla V100 GPU. In all experiments, we used the AdamW optimizer~\cite{Loshchilov2017DecoupledWD} with learning rate \(10^{-3}\), weight decay \(10^{-6}\), batch size \(128\), and trained for \(100{,}000\) iterations.

\subsection{Overview of iGSM data and experimental setup}
\label{app:iGSM}
The iGSM dataset~\cite{YXLA2024-gsm1} is a synthetic dataset designed to mimic the structural patterns and reasoning processes of gsm8k~\cite{cobbe2021gsm8k}. It is generated via a programmatic pipeline that constructs problems from an underlying dependency graph, where nodes represent parameters and edges encode compositional relationships. Each generated instance consists of a problem statement, a step-by-step solution, and a final answer, all derived from the same latent structure. This design enables precise control over problem complexity and reasoning depth. 

We adopt iGSM to study how data influences frequency usage, as it is both scalable--making it suitable for mutual information estimation and pretraining--and flexible, allowing fine-grained control over the effective receptive field by adjusting the number of operations (ops) used during data generation. An example from iGSM is shown in~\Cref{fig:example_iGSM}.

\begin{figure*}[htbp]
    \centering
    \begin{tcolorbox}[title=iGSM dataset]
    \small
    \textbf{Problem:}\\
    The number of each Penguin Beach's Giraffe equals 6. The number of each Octopus Den's Leopard equals each Octopus Den's Giraffe. The number of each Rockpool Exhibit's Leopard equals 20 more than the sum of each Octopus Den's Giraffe and each Octopus Den's Leopard. The number of each Rockpool Exhibit's Giraffe equals 8 times the sum of each Octopus Den's Giraffe and each Octopus Den's Leopard. The number of each Octopus Den's Giraffe equals 21. How many Animals does Penguin Beach have?

    \vspace{0.6em}
    \textbf{Answer:} 6

    \vspace{0.6em}
    \textbf{Solution:}\\
    Let the number of Giraffes at Penguin Beach be $e$, so $e = 6$. Let the number of Animals at Penguin Beach be $J$, then $J = e = 6$.
    \end{tcolorbox}
    \caption{An example from the iGSM dataset}
    \label{fig:example_iGSM}
\end{figure*}

In our controlled setting, we generate datasets with a fixed number of operations. For example, iGSM-2 denotes data generated with ops $= 2$. We fix the number of edges in the underlying structure to $\left\lfloor \text{ops} \cdot \frac{3}{2} \right\rfloor + 1$. As the number of operations increases, both the problem descriptions and the corresponding reasoning traces become longer, resulting in broader dependency width. We directly concatenate the question with the solution, separated by “\verb|\n\n|”.

\paragraph{Training details} All models were trained on a NVIDIA L40 GPU. We pretrained 12-layer GPT-style Transformer models using the nanochat \cite{nanochat} framework. In each experiment, we followed the default data scaling configuration provided by the framework: a batch size of 524{,}288 tokens is used for 2{,}205 steps, corresponding to approximately 1.2B tokens in total. The context length was set to 2048, and we used a BOS-aligned dataloader with Best-Fit Cropping. We used $\theta_{base}=100,000$. All other hyperparameters were kept at their default values.

\subsection{Measuring mutual information}
\label{app:MI}
We estimate mutual information (MI) on natural language using the ClimbMix-400B dataset \cite{diao2025climb} by sampling several random archors $i$, computing empirical distributions $p(x_i), p(x_j), p(x_i, x_j)$ over all documents and applying the standard definition of MI. To reduce bias in finite-sample estimation, we apply the Miller--Madow correction \cite{miller1955note}.

To study the effect of tokenizer vocabulary size, we train a separate BPE tokenizer on the ClimbMix data for each chosen vocabulary size. We then tokenize the same set of documents with each tokenizer and compute the corresponding MI estimates.

When evaluating on iGSM, we follow the same procedure with one modification: we realign each example using the problem–solution boundary and restrict anchor positions to tokens in the solution and target positions $j$ to tokens in the problem. This setup focuses on dependencies from the problem to the solution. We then normalize the estimated mutual information into a probability distribution and compute its cumulative distribution function (CDF), which is used to construct Q–Q plots for comparison across datasets.

\subsection{Testing position interpolation on arithmetic tasks}
\label{app:PI_arithmetic}
We construct a synthetic arithmetic dataset to evaluate position interpolation in non self-similar task.

\paragraph{Data generation}
Each sample is generated as follows:
\begin{itemize}
\item[(1)] uniformly sample an operator $\text{op} \in \{+, -, \times, \div\}$(where $\div$ denotes integer division);

\item[(2)] uniformly sample the number of digits $\text{digit}_a, \text{digit}_b \in \{1,2,3,4,5\}$;

\item[(3)] uniformly sample two integers $a$ and $b$ with $\text{digit}_a$ and $\text{digit}_b$ digits, respectively.
\end{itemize}
Let $s$ denote the corresponding correct result. Each example is then formatted as:
\[
\texttt{a op b = s.}
\]
We concatenate independently generated samples into a single sequence until reaching a target context length $L \in \{2048,4096,8192,16384\}$. This produces long sequences containing multiple arithmetic expressions. 

\paragraph{Evaluation protocol}
The evaluation was done on a NVIDIA Tesla T4 GPU. 
To isolate positional effects and avoid error accumulation from autoregressive decoding, we did not perform step-by-step generation. Instead, we ran a single forward pass over the entire concatenated sequence. For each answer span, a prediction was counted as correct only if the model's argmax output matched the ground-truth token at every position in that span (i.e., exact match over the full answer).
We mapped each answer span to its absolute position in the sequence and partitioned the context range into disjoint buckets. For each bucket, we computed the average accuracy over all answer spans whose tokens fell into that bucket. The final results are averaged over 100 random sequences. 

\section{Examples of $K(r)$}

We now instantiate the positional dependency kernel \(K(r)\) in several settings considered in the work.

\paragraph{Block-structured data}
Consider block-structured data in~\Cref{sec:motivation_training_data}:
\[
    x_t=z_{{block}(t)}+\epsilon_t,
\]
where each block has length \(B\), each block mean satisfies
\[
    \mu_b\in\{-1,+1\},
\]
and \(\epsilon_t\) is independent unit Gaussian noise. Two positions at distance \(r\) share the same block mean only if they lie in the same block. Averaging over all pairs at distance \(r\), the probability of being in the same block is approximately
\[
    \left(1-\frac{r}{B}\right)_+.
\]
Thus the autocorrelation or relevance kernel has the triangular form
\[
    K_B(r)
    \propto
    \left(1-\frac{r}{B}\right)_+.
\]
Its support is \([0,B]\), so \(W=B\) and \Cref{thm:field_constrained_optimal_frequency} predicts
\[
    \theta^\star
    =
    \frac{\pi}{B}.
\]
Thus increasing the block length shifts useful RoPE frequencies toward lower values.

\paragraph{Concatenated question-answer data}
Consider sequences of the form
\[
    q_1a_1q_2a_2\cdots q_na_n,
\]
where each answer \(a_i\) depends on the preceding question span \(q_i\). If the question length is \(Q\), then for predicting \(a_i\), the relevant information lies over a window of width approximately \(Q\). Therefore the positional dependency kernel is approximately
\[
    K_Q(r)
    \approx
    \mathbf 1\{0\le r\le Q\}.
\]
Its field width is \(W=Q\), giving
\[
    \theta^\star
    \asymp
    \frac{1}{Q}.
\]
Thus longer question spans induce lower-frequency RoPE usage.

\paragraph{Key-value lookup}
Consider associative key-value lookup data
\[
    (x_1,y_1,x_2,y_2,\ldots,x_n,y_n),
\]
where a query token \(x\) must retrieve the corresponding value \(y_i\). The total search range may grow with \(n\), but the position-dependent part of the computation is the local binding relation
\[
    x_i\mapsto y_i.
\]
If \(y_i\) immediately follows \(x_i\), then the relevant positional dependency kernel is concentrated at offset \(1\):
\[
    K_{\operatorname{bind}}(r)
    \approx
    \delta_1(r).
\]
Thus the positional field width is small, \(W\approx 1\), and the theory predicts high-frequency usage.

This explains why key-value lookup tasks can induce high-frequency RoPE energy even when the context contains many pairs: content identifies the matching key, while position is primarily used to bind each key to its adjacent value. The relevant positional scale is the key-value gap, not the number of pairs.

If the key-value gap is increased to \(g\), so that the value lies \(g\) tokens after the key, then
\[
    K_{\operatorname{bind}}(r)
    \approx
    \delta_g(r),
\]
and the predicted frequency scale decreases as
\[
    \theta^\star
    \asymp
    \frac{1}{g}.
\]

\paragraph{Natural language as a mixture of fields}
Natural language does not generally have a single dependency scale. It contains local collocations, phrase-level structure, clause-level dependencies, sentence-level coherence, entity tracking, and discourse-level dependencies. Thus its positional dependency kernel is better viewed as a mixture over field sizes:
\[
    K_{\operatorname{NL}}(r)
    =
    \int K_W(r)\,d\nu(W),
\]
where \(K_W\) is a dependency kernel of field width \(W\), and \(\nu\) is a mixing distribution over scales.

Under the inverse-frequency principle
\[
    \theta\asymp \frac{1}{W},
\]
a mixture over field sizes could produce a mixture over useful frequencies. This explains why real language models can exhibit mid-low frequency usage band rather than a single sharp peak. Local \(n\)-gram structure may require some high-frequency components, while broader phrase, sentence, and discourse dependencies may dominate the aggregate RoPE spectrum.

\section{Additional Experimental Results}
\label{app:exp_add}

\begin{figure}[h]
\centering
\includegraphics[width=\linewidth]{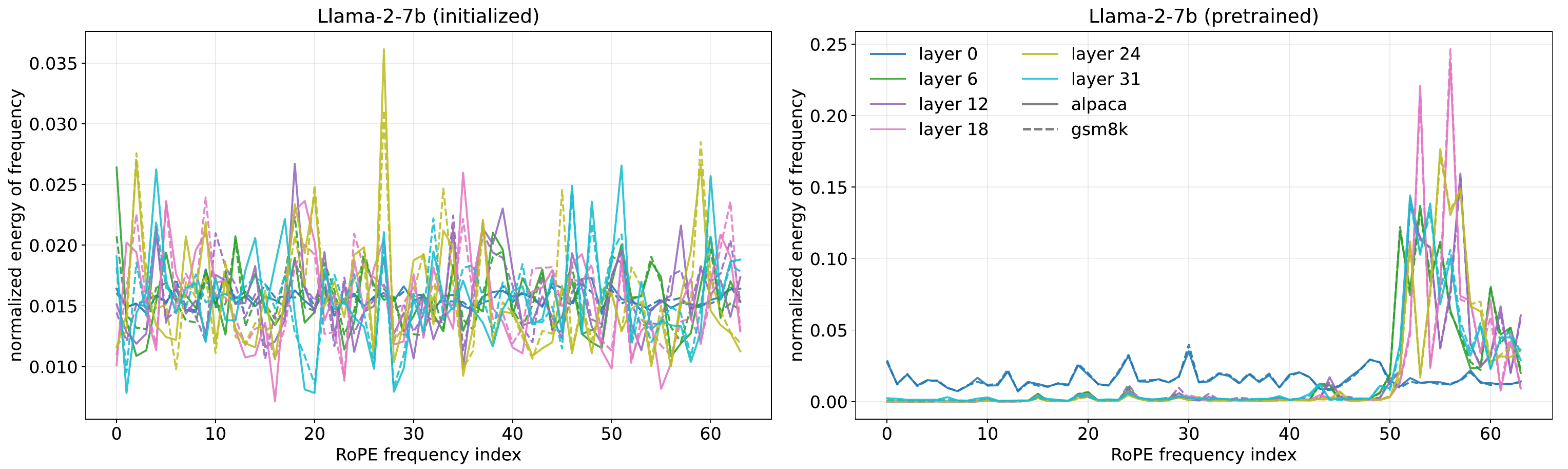}
\caption{\small Frequency energy (averaged over heads) of Llama-2-7B model tested on alpaca and gsm8k datasets. The frequency usage concentrates on low frequencies after training and remains stable across prompts. }
\label{fig:input_stable_llama}
\end{figure}

\subsection{RoPE frequency spectra are stable across model families}
\label{app:input_stable_llama}

We repeat the frequency-energy analysis from~\Cref{sec:motivation_training_data} on Llama-2-7B~\cite{touvron2023llama}. Specifically, we compute \(E_m\) on Alpaca~\cite{alpaca} and GSM8K~\cite{cobbe2021gsm8k}; the results are shown in~\Cref{fig:input_stable_llama}. The same qualitative pattern observed in~\Cref{fig:input_stable} also appears here. After training, the RoPE frequency spectrum is highly non-uniform, with most energy concentrated in a restricted range of low frequencies. Moreover, the spectra are largely stable across inputs and across domains, suggesting that RoPE frequency usage is encoded in the learned query-key weights rather than selected independently for each prompt. These results indicate that learned, input-stable RoPE frequency usage is not specific to a single model family. For additional examples across model families, see~\citet{Jin2025MassiveVI}.

\subsection{Perplexity and long-context arithmetic evaluation}
\label{app:exp_add_ppl}

We further evaluate whether perplexity reliably reflects long-context ability in the arithmetic task. The model we use, Llama-2-7B~\cite{touvron2023llama}, is pretrained with a context length of \(4096\). During evaluation, we test the original model and PI-extended variants with context lengths \(8192\) and \(16384\), and report perplexity over different evaluation ranges, following the protocol of~\citet{chen2023extending}. 

\begin{table}[h]
\centering
\small
\caption{Perplexity of Llama-2-7B with PI under different context lengths on arithmetic task. The pretraining context length is 4096. }
\begin{tabular}{ccccc}
\toprule
& \multicolumn{4}{c}{Evaluation Context Window Range}\\
Context Window & [0,2048) & [2048,4096) & [4096,8192) & [8192,16384) \\
\midrule
4096 (No PI)  & 2.09  & 2.05 & >100 & >100 \\
8192 (PI)  & 5.26  & 4.70 & 4.39 & -- \\
16384 (PI) & 18.37  & 16.54 & 15.80 & 15.38 \\

\bottomrule
\end{tabular}
\label{tab:pi_ppl_arithmetic}
\end{table}

\begin{table}[h]
\centering
\small
\caption{Perplexity of Llama-7B~\cite{touvron2023llama1} on natural language, reported in Table~1 of \citet{chen2023extending}. The pretraining context length is 2048. }
\begin{tabular}{ccccc}
\toprule
& \multicolumn{4}{c}{Evaluation Context Window Range}\\
Context Window & [0,2048) & [2048,4096) & [4096,8192) & [8192,16384)\\
\midrule
2048 (No PI)  & 7.20  & >100 & >100 & >100 \\
8192 (PI)  & 7.13  & 6.96 & 6.95 & -- \\
16384 (PI) & 7.11  & 6.93 & 6.82 & 6.83 \\

\bottomrule
\end{tabular}
\label{tab:pi_ppl_NL}
\end{table}
As shown in~\Cref{tab:pi_ppl_arithmetic}, without PI, perplexity degrades sharply once the context length exceeds the pretraining length. In contrast, PI substantially lowers perplexity at longer context lengths. For example, at context length \(8192\), PI reduces perplexity from over \(100\) to \(4.39\), and at context length \(16384\), it reduces perplexity from over \(100\) to \(15.38\).

However, lower perplexity does not necessarily imply successful long-context reasoning. In this arithmetic task, exact answer accuracy remains low even when PI improves perplexity. Inspecting the answer tokens reveals the issue: at context length \(16384\), the answer-token perplexity under PI is approximately \(16\), which is much larger than the original model and even worse than the perplexity \(10\) of a uniform guess over digits \(0\)--\(9\). This contrasts sharply with natural language evaluation in~\Cref{tab:pi_ppl_NL}, where PI preserves perplexity close to the original model, suggesting that the underlying capability is largely retained. Thus, although PI improves the overall likelihood assigned by the model, it does not necessarily recover the precise answer.

These results show that perplexity can overstate long-context ability. This is especially problematic for tasks that require precise positional resolution, where assigning moderately high likelihood to locally plausible tokens is insufficient. Our findings are consistent with previous observations that perplexity may be an unreliable proxy for long-context generalization~\cite{fang2024wrong}.

\section{Licenses}

\textbf{Libraries}
\begin{itemize}
    \item transformer \cite{wolf-etal-2020-transformers}: Apache License 2.0
    \item nanochat \cite{nanochat}: MIT License
\end{itemize}

\textbf{Datasets}
\begin{itemize}
    \item ClimbMix-400B \cite{diao2025climb}: CC BY-NC 4.0
    \item iGSM \cite{YXLA2024-gsm1}: MIT License
\end{itemize}

\textbf{Models}
\begin{itemize}
    \item Qwen2.5 series \cite{qwen2.5}: Apache License 2.0
    \item Llama2-7B \cite{touvron2023llama}: Llama 2 Community License Agreement
    \item Llama-7B \cite{touvron2023llama1}: Llama License Agreement
\end{itemize}
\newpage

\end{document}